\documentclass[runningheads,a4paper]{llncs}

\usepackage{amssymb}

\usepackage{algorithm}
\usepackage{algorithmic}

\usepackage{graphicx}
\usepackage{subfigure}

\usepackage{url}

\def\D{{\bf D}}

\def\H{{\bf H}}
\def\I{{\bf I}}

\def\R{{\bf R}}
\def\S{{\bf S}}

\def\tt{{\bf t}}
\def\U{{\bf U}}
\def\u{{\bf u}}
\def\V{{\bf V}}
\def\v{{\bf v}}
\def\W{{\bf W}}

\def\X{{\bf X}}
\def\x{{\bf x}}
\def\Y{{\bf Y}}
\def\y{{\bf y}}

\def\z{{\bf z}}
\def\0{{\bf 0}}
\def\1{{\bf 1}}

\def\CM{{\mathcal C}}

\def\OM{{\mathcal O}}

\def\WM{{\mathcal W}}

\def\RB{{\mathbb R}}

\def\EB{{\mathbb E}}
\def\PB{{\mathbb P}}

\def\bet{\mbox{\boldmath$\beta$\unboldmath}}

\def\Ph{\mbox{\boldmath$\Phi$\unboldmath}}

\def\Si{\mbox{\boldmath$\Sigma$\unboldmath}}

\def\argmin{\mathop{\rm argmin}}

\def\lsr{\textrm{lsr}}

\def\rk{\mathrm{rank}}
\def\diag{\mathsf{diag}}

\begin{document}

\mainmatter  

\title{Sharpened Error Bounds for Random Sampling Based $\ell_2$ Regression}

\titlerunning{Error Bound for Random Sampling Based $\ell_2$ Regression}

%
%

\author{Shusen Wang}
\institute{College of Computer Science \& Technology, Zhejiang Unversity\\
wss@zju.edu.cn}
%
\authorrunning{Shusen Wang}
%
%

\toctitle{Sharpened Error Bound for the Random Sampling Based $\ell_2$ Regression}
\tocauthor{Wang}
\maketitle

\begin{abstract}
Given a data matrix $\X \in \RB^{n\times d}$ and a response vector $\y \in \RB^{n}$,
suppose $n>d$,
it costs $\OM(n d^2)$ time and $\OM(n d)$ space to solve
the least squares regression (LSR) problem.
When $n$ and $d$ are both large, exactly solving the LSR problem is very expensive.
When $n \gg d$, one feasible approach to speeding up LSR is to randomly embed $\y$ and all columns of $\X$ into a smaller subspace $\RB^c$;
the induced LSR problem has the same number of columns but much fewer number of rows,
and it can be solved in $\OM(c d^2)$ time and $\OM(c d)$ space.

We discuss in this paper two random sampling based methods for solving LSR more efficiently.
Previous work showed that the leverage scores based sampling based LSR achieves $1+\epsilon$ accuracy when $c \geq \OM(d \epsilon^{-2} \log d)$.
In this paper we sharpen this error bound, showing that $c = \OM(d \log d + d \epsilon^{-1})$ is enough for achieving $1+\epsilon$ accuracy.
We also show that when $c \geq \OM(\mu d \epsilon^{-2} \log d )$, the uniform sampling based LSR attains a $2+\epsilon$ bound with positive probability.
\end{abstract}

\section{Introduction} \label{sec:introduction}

Given $n$ data instances $\x^{(1)}, \cdots , \x^{(n)}$, each of dimension $d$, and $n$ responses $y_1, \cdots , y_n$,
it is interesting to find a model $\bet \in \RB^d$ such that $\y = \X \bet$.
If $n > d$, there will not in general exist a solution to the linear system,
so we instead seek to find a model $\bet_{\lsr}$ such that $\y \approx \X \bet_\lsr$.
This can be formulated as the least squares regression (LSR) problem:
\begin{equation} \label{eq:least_square}
\bet_{\lsr} \; = \; \argmin_{\bet \in \RB^{d}} \big\| \y - \X \bet\big\|_2^2 .
\end{equation}
Suppose $n\geq d$, in general it takes $\OM(n d^2)$ time and $\OM(n d)$ space to compute $\bet_{\lsr}$
using the Cholesky decomposition, the QR decomposition, or the singular value decomposition (SVD) \cite{golub1996matrix}.

LSR is perhaps one of the most widely used method in data processing,
however, solving LSR for big data is very time and space expensive.
In the big-data problems where $n$ and $d$ are both large,
the $\OM(n d^2)$ time complexity and $\OM(n d)$ space complexity make LSR prohibitive.
So it is of great interest to find efficient solution to the LSR problem.
Fortunately, when $n \gg d$, one can use a small portion of the $n$ instances instead of using the full data to approximately compute $\bet_{\lsr}$,
and the computation cost can thereby be significantly reduced.
Random sampling based methods \cite{drineas2008cur,ma2014statistical,mahoney2011ramdomized}
and random projection based methods \cite{clarkson2013low,drineas2011faster}
have been applied to make LSR more efficiently solved.

Formally speaking, let $\S \in \RB^{c \times n}$ be a random sampling/projection matrix,
we solve the following problem instead of (\ref{eq:least_square}):
\begin{equation} \label{eq:approximate_least_square}
\tilde{\bet}_{\S} \; = \; \argmin_{\bet \in \RB^{d}} \big\| \S \y - \S \X \bet\big\|_2^2 .
\end{equation}
This problem can be solved in only $\OM(c d^2)$ time and $\OM(c d)$ space.
If the random sampling/projection matrix $\S$ is constructed using some special techniques,
then it is ensured theoretically that $\tilde{\bet}_{\S} \approx \bet_{\lsr}$
and that
    \begin{eqnarray} \label{eq:1+epsilon_bound}
        \big\| \y - \X \tilde{\bet}_\S \big\|_2^2 \; \leq \; (1+\epsilon)\, \big\| \y - \X \bet_{\lsr} \big\|_2^2
    \end{eqnarray}
hold with high probability.
There are two criteria to evaluate random sampling/projection based LSR.
\begin{itemize}
\item {\bf Running Time.}
    That is, the total time complexity in constructing $\S \in \RB^{c\times n}$ and computing $\S \X \in \RB^{c \times d}$.
\item {\bf Dimension after Projection.}
    Given an error parameter $\epsilon$, we hope that there exists a polynomial function $C(d, \epsilon)$ such that if $c > C(d, \epsilon)$,
    the inequality (\ref{eq:1+epsilon_bound}) holds with high probability for all $\X \in \RB^{n\times d}$ and $\y\in \RB^{n}$.
    Obviously $C(d, \epsilon)$ is the smaller the better because the induced problem (\ref{eq:approximate_least_square})
    can be solved in less time and space if $c$ is small.
\end{itemize}

\subsection{Contributions}

The leverage scores based sampling is an important random sampling technique widely studied and empirically evaluated in the literature
\cite{clarkson2013low,drineas2006sampling,drineas2008cur,gittens2013revisiting,ma2014statistical,mahoney2009matrix,mahoney2011ramdomized,wang2013improving,wang2014efficient}.
When applied to accelerate LSR,
error analysis of the leverage scores based sampling is available in the literature.
It was shown in \cite{drineas2006sampling} showed that by using the leverage scores based sampling with replacement, when
\[
c \geq \OM\big( d^2 \epsilon^{-2} \big),
\]
the inequality (\ref{eq:1+epsilon_bound}) holds with high probability.
Later on, \cite{drineas2008cur} showed that by using the leverage scores based sampling without replacement,
\[
c \geq \OM\big( d \epsilon^{-2}  \log d \big)
\]
is sufficient to make (\ref{eq:1+epsilon_bound}) hold with high probability.
In Theorem~\ref{thm:leverage_score_sampling} we show that (\ref{eq:1+epsilon_bound}) holds with high probability when
\[
c \geq \OM\big( d \log d + d \epsilon^{-1} \big),
\]
using the same leverage scores based sampling without replacement.
Our results are described in Theorem~\ref{thm:leverage_score_sampling}.
Our proof techniques are based on the previous work \cite{drineas2008cur,drineas2011faster},
and our proof is self-contained.

Though uniform sampling can have very bad worst-case performance when applied to the LSR problem \cite{ma2014statistical},
it is still the simplest and most efficient strategy for column sampling.
Though the uniform sampling is uniformly worse than the leverage score based sampling from an algorithmic perspective,
it is not true from a statistical perspective \cite{ma2014statistical}.
Furthermore, when the leverage scores of $\X$ are very homogeneous,
the uniform sampling has virtually the same performance as the leverage scores based sampling \cite{ma2014statistical}.
So uniform sampling is still worth of study.
We provide in Theorem~\ref{thm:uniform_sampling} an error bound for the uniform sampling based LSR.
We show that when
\[
c > \OM\big(\mu d \epsilon^{-2} \log d \big),
\]
the uniform sampling based LSR attains a $2+\epsilon$ bound with positive probability.
Here $\mu$ denotes the matrix coherence of $\X$.

\begin{algorithm}[tb]
   \caption{The Leverage Scores Based Sampling (without Replacement).}
   \label{alg:lev_sampling}
\algsetup{indent=2em}
\begin{small}
\begin{algorithmic}[1]
   \STATE {\bf Input:} an $n\times d$ real matrix $\X$, target dimension $c < n$.
   \STATE (Exactly or approximately) compute the leverage scores of $\X$: $l_1 , \cdots , l_n$;
   \STATE Compute the sampling probabilities by $p_i = \min \{1, c l_i / d \}$ for $i = 1$ to $n$;
   \STATE Denote the set containing the indices of selected rows by $\CM$, initialized by $\emptyset$;
   \STATE For each index $i \in [n]$, add $i$ to $\CM$ with probability $p_i$;
   \STATE Compute the diagonal matrix $\D = \diag (p_1^{-1}, \cdots, p_n^{-1})$;
   \RETURN $\S \longleftarrow$ the rows of $\D$ indexed by $\CM$.
\end{algorithmic}
\end{small}
\end{algorithm}

\section{Preliminaries and Previous Work}

For a matrix $\X = [x_{ij}] \in \RB^{n\times d}$,
we let $\x^{(i)}$ be its $i$-th row,
$\x_j$ be its $j$-th column,
$\|\X\|_F = \big( \sum_{i,j} x_{ij}^2 \big)^{1/2}$ be its Frobenius norm,
and $\|\X\|_2 = \max_{\|\z\|_2=1} \|\X \z\|_2$ be its spectral norm.
We let $\I_n$ be an $n\times n$ identity matrix and let $\0$ be an all-zero matrix with proper size.

We let the thin singular value decomposition of $\X\in \RB^{n\times d}$ be
\[
\X
\;=\; \U_\X \Si_\X \V_\X^T
\;=\; \sum_{i=1}^{d} \sigma_i (\X) \, \u_{\X,i} \, \v_{\X,i}^T .
\]
Here $\U_\X$, $\Si_\X$, and $\V_\X$ are of sizes $n\times d$, $d\times d$, and $d\times d$,
and the singular values $\sigma_1 (\X) , \cdots , \sigma_d (\X)$ are in non-increasing order.
We let $\U_\X^\perp$ be an $n \times (n-d)$ column orthogonal matrix such that $\U_\X^T \U_\X^\perp = \0$.
The condition number of $\X$ is defined by $\kappa (\X) = \sigma_{\max} (\X) / \sigma_{\min} (\X)$.

Based on SVD, the (row) {\it statistical leverage scores} of $\X\in \RB^{n\times d}$ is defined by
\begin{eqnarray}
l_i = \big\| \u_\X^{(i)} \big\|_2^2 \textrm{, } \quad i = 1 , \cdots , n \textrm{,} \nonumber
\end{eqnarray}
where $\u_\X^{(i)}$ is the $i$-th row of $\U_\X$.
It is obvious that $\sum_{i=1}^n l_i = d$.
Exactly computing the $n$ leverages scores costs $\OM(n d^2)$ time,
which is as expensive as exactly solving the LSR problem (\ref{eq:least_square}).
Fortunately, if $\X$ is a skinny matrix,
the leverages scores can be highly efficiently computed within arbitrary accuracy using the techniques of \cite{clarkson2013low,drineas2012fast}.

There are many ways to construct the random sampling/projection matrix $\S$, and below we describe some of them.
\begin{itemize}
\item {\bf Uniform Sampling.}
    The sampling matrix $\S$ is constructed by sampling $c$ rows of the identity matrix $\I_n$ uniformly at random.
    This method is the simplest and fastest, but in the worst case its performance is very bad \cite{ma2014statistical}.
\item {\bf Leverage Scores Based Sampling.}
    The sampling matrix $\S$ is computed by Algorithm~\ref{alg:lev_sampling};
    $\S$ has $c$ rows in expectation.
    This method is proposed in \cite{drineas2006sampling,drineas2008cur}.
\item {\bf Subsampled Randomized Hadamard Transform (SRHT).}
    The random projection matrix $\S = \sqrt{n/c} \R \H \D $ is called SRHT \cite{boutsidis2013improved,drineas2011faster,tropp2011improved}
    if
    \begin{itemize}
    \item
        $\R \in \RB^{c \times n}$ is a subset of $c$ rows from the $n \times n$ identity matrix,
        where the rows are chosen uniformly at random and without replacement;
    \item
        $\H \in \RB^{n \times n}$ is a normalized Walsh--Hadamard matrix;
    \item
        $\D$ is an $n \times n$ random diagonal matrix with each diagonal entry independently chosen to be $+1$ or $-1$ with equal probability.
    \end{itemize}
    SRHT is a fast version of the Johnson-Lindenstrauss transform.
    The performance of SRHT based LSR is analyzed in \cite{drineas2011faster}.
\item {\bf Sparse Embedding Matrices.}
    The sparse embedding matrix $\S = \Ph \D$ enables random projection performed
    in time only linear in the number of nonzero entries of $\X$ \cite{clarkson2013low}.
    The random linear map $\S = \Ph \D$ is defined by
    \begin{itemize}
    \item
        $h: [n] \mapsto [c]$ is a random map so that for each $i\in [n]$, $h(i) = t$ for $t \in [c]$ with probability $1 / c$;
    \item
        $\Ph \in \{0,1\}^{c \times n}$ is a $c\times n$ binary matrix with $\Phi_{h(i), i} =1$, and all remaining entries $0$;
    \item
        $\D$ is the same to the matrix $\D$ of SRHT.
    \end{itemize}
    Sparse embedding matrices based LSR is guaranteed theoretically in \cite{clarkson2013low}.
\end{itemize}

\section{Main Results}

We provide in Theorem \ref{thm:leverage_score_sampling}
an improved error bound for the leverage scores sampling based LSR.

\begin{theorem} [The Leverage Score Based Sampling]
\label{thm:leverage_score_sampling}
Use the {\it leverage score based sampling without replacement} (Algorithm~\ref{alg:lev_sampling}) to construct the ${c\times d}$ sampling matrix $\S$ where
\[
c \geq \OM(d \ln d + d \epsilon^{-1}),
\]
and solve the approximate LSR problem (\ref{eq:approximate_least_square}) to obtain $\tilde{\bet}_\S$.
Then with probability at least $0.8$ the following inequalities hold:
\begin{eqnarray*}
\big\| \y - \X \tilde{\bet}_\S \big\|_2^2
& \;\leq\; & (1+\epsilon)\, \big\| \y - \X \bet_{\lsr} \big\|_2^2, \\
\big\| \bet_{\lsr} - \tilde{\bet}_\S \big\|_2^2
& \;\leq\; & \frac{ \epsilon }{ \sigma_{\min}^2 (\X) } \big\| \y - \X \bet_{\lsr} \big\|_2^2
\; \leq \;
\epsilon \, \kappa^2 (\X) \: \big(\gamma^{-2} - 1 \big) \: \big\|\bet_{\lsr} \big\|_2^2 ,
\end{eqnarray*}
where $\gamma$ is defined by $\gamma \leq \|\U_\X \U_\X^T \y \|_2  \, / \, \|\y\|_2 \leq 1$.
\end{theorem}

We show in Theorem~\ref{thm:uniform_sampling} an error bound for the uniform sampling based LSR.

\begin{theorem}[Uniform Sampling] \label{thm:uniform_sampling}
Use the uniform sampling without replacement to sample
\[
c \geq 1000 \mu d (\ln d + 7)
\]
rows of $\X$ and compute the approximate LSR problem (\ref{eq:approximate_least_square}) to obtain $\tilde{\bet}_\S$.
Then with probability $ 0.05 $ the following inequalities hold:
\begin{eqnarray*}
\| \y - \X \tilde{\bet}_\S \|_2^2 & \;\leq\; & 2.2 \, \| \y - \X \bet_{\lsr} \|_2^2, \\
\big\| \bet_{\lsr} - \tilde{\bet}_\S \big\|_2^2
& \;\leq\; & \frac{ 1.2 }{ \sigma_{\min}^2 (\X) } \big\| \y - \X \bet_{\lsr} \big\|_2^2
\; \leq \;
1.2 \, \kappa^2 (\X) \: \big(\gamma^{-2} - 1 \big) \: \big\|\bet_{\lsr} \big\|_2^2 ,
\end{eqnarray*}
where $\gamma$ is defined by $\gamma \leq \|\U_\X \U_\X^T \y \|_2  \, / \, \|\y\|_2 \leq 1$.
\end{theorem}

Since computing $\| \y - \X \tilde{\bet}_{\S}\|_2$ costs only $\OM(nd)$ time,
so one can repeat the procedure $t$ times and choose the solution that attains the minimal error $\| \y - \X \tilde{\bet}_{\S}\|_2$.
In this way, the error bounds hold with probability $1 - 0.95^t$ which can be arbitrarily high.

\section{Proof}

In Section~\ref{sec:lemmas} we list some of the previous work that will be used in our proof.
In Section~\ref{sec:proof:thm:leverage_score_sampling} we prove Theorem~\ref{thm:leverage_score_sampling}.
We prove the theorem by using the techniques in the proof of Lemma 1 and 2 of \cite{drineas2008cur} and Lemma 1 and 2 of \cite{drineas2011faster}.
For the sake of self-contain, we repeat some of the proof of \cite{drineas2008cur} in our proof.
In Section~\ref{sec:proof:thm:uniform_sampling} we prove Theorem~\ref{thm:uniform_sampling}
using the techniques in \cite{drineas2011faster,gittens2011spectral,tropp2011improved}.

\subsection{Key Lemmas} \label{sec:lemmas}

Lemma~\ref{lem:lem1_drineas2011faster} is a deterministic error bound for the sampling/projection based LSR,
which will be used to prove both of Theorem~\ref{thm:leverage_score_sampling} and Theorem~\ref{thm:uniform_sampling}.
The random matrix multiplication bounds in Lemma~\ref{lem:thm7_drineas2008cur} will be used to prove Theorem~\ref{thm:leverage_score_sampling}.
The matrix variable tail bounds in Lemma~\ref{lemma:matrix_chernoff} will be used to prove Theorem~\ref{thm:uniform_sampling}.

\begin{lemma}[Deterministic Error Bound, Lemma 1 and 2 of \cite{drineas2011faster}] \label{lem:lem1_drineas2011faster}
Suppose we are given an overconstrained least squares approximation problem with $\X \in \RB^{n\times d}$ and $\y \in \RB^{n}$.
We let $\bet_{\lsr}$ be defined in (\ref{eq:least_square})
and $\tilde{\bet}_{\S}$ be defined in (\ref{eq:approximate_least_square}),
and define ${\z}_{\S} \in \RB^{d}$ such that $\U_\X \z_{\S} = \X (\bet_{\lsr} - \tilde{\bet}_{\S})$.
Then the following equality and inequalities hold deterministically:
\begin{eqnarray*}
\big\| \y - \X \tilde{\bet}_\S \big\|_2^2 & \;=\; & \big\| \y - \X \bet_{\lsr} \big\|_2^2 + \big\| \U_\X \z_\S \big\|_2^2 , \\
\big\| \bet_{\lsr} - \tilde{\bet}_\S \big\|_2^2 & \;\leq\; & \frac{ \| \U_\X \z_\S \|_2^2 }{ \sigma_{\min}^2 (\X) } ,\\
\big\|\z_{\S} \big\|_2 & \;\leq\; & \frac{ \big\| \U_\X^T \S^T \S \U_\X^{\perp} {\U_\X^{\perp}}^T \y \big\|_2 }{ \sigma_{\min}^2 (\S \U_\X) } .
\end{eqnarray*}
By further assuming that $\|\U_\X \U_\X^T \y \|_2 \geq \gamma \|\y\|_2$,
it follows that
\begin{eqnarray*}
\big\| \U_\X^{\perp} {\U_\X^{\perp}}^T \y \big\|_2^2
\; \leq \; \sigma^2_{\max} (\X) \: \big(\gamma^{-2} - 1\big) \: \big\|\bet_{\lsr} \big\|_2^2 .
\end{eqnarray*}
\end{lemma}

\begin{proof}
The equality and the first two inequalities follow from Lemma 1 of \cite{drineas2011faster}.
The last inequality follows from Lemma 2 of \cite{drineas2011faster}.
\end{proof}

\begin{lemma}[Theorem 7 of \cite{drineas2008cur}] \label{lem:thm7_drineas2008cur}
Suppose $\X \in \RB^{d\times n}$, $\Y \in \RB^{n\times p}$, and $c \leq n$,
and we let $\S \in \RB^{c\times n}$ be the sampling matrix computed by Algorithm~\ref{alg:lev_sampling} taking $\X$ and $c$ as input,
then
\begin{eqnarray}
\EB \big\| \X^T \Y - \X^T \S^T \S \Y  \big\|_F
&\; \leq \;&
\frac{1}{\sqrt{c}} \:  \big\| \X \big\|_F \:  \big\| \Y \big\|_F , \nonumber \\
\EB \big\| \X^T \X - \X^T \S^T \S \X  \big\|_F
&\; \leq \;&
\OM \bigg(\sqrt{ \frac{\log c}{c}} \bigg) \:  \big\| \X \big\|_2 \:  \big\| \X \big\|_F . \nonumber
\end{eqnarray}
\end{lemma}

\begin{lemma}[Theorem 2.2 of \cite{tropp2011improved}] \label{lemma:matrix_chernoff}
Let $\WM$ be a finite set of positive semidefinite matrices with dimension $d$, and suppose that
\[
\max_{\W \in \WM} \; \lambda_{\max} (\W) \leq R \textrm{.}
\]
Sample $\W_1 ,\cdots , \W_c$ uniformly at random from $\WM$ without replacement.
We define  $\xi_{\min} = c \lambda_{\min}\big( \EB \W_1\big)$ and
$\xi_{\max} = c \lambda_{\max}\big( \EB \W_1\big)$.
Then for any $\theta_1 \in (0, 1]$ and $\theta_2 > 1$, the following inequalities hold:
\begin{eqnarray}
\PB \bigg\{ \lambda_{\min} \Big(\sum_{i=1}^c \W_i \Big) \leq \theta_1 \xi_{\min}  \bigg\}
& \; \leq \; &
d \bigg[ \frac{e^{\theta_1 -1}}{\theta_1^{\theta_1}} \bigg]^{{\xi_{\min}}/{R}}, \nonumber \\
\PB \bigg\{ \lambda_{\max} \Big(\sum_{i=1}^c \W_i \Big) \geq \theta_2 \xi_{\max}  \bigg\}
& \; \leq \; &
d \bigg[ \frac{e^{\theta_2 -1}}{\theta_2^{\theta_2}} \bigg]^{{\xi_{\max}}/{R}}. \nonumber
\end{eqnarray}
\end{lemma}

\subsection{Proof of Theorem~\ref{thm:leverage_score_sampling}} \label{sec:proof:thm:leverage_score_sampling}

\begin{proof}
We first bound the term $\sigma_{\min}^2$ as follows.
Applying a singular value inequality in \cite{horn1991topics}, we have that for all $i \leq \rk (\X)$
\begin{eqnarray}
\big| 1 - \sigma_i^2 (\S \U_\X) \big|
&\; = \;&
\big| \sigma_i \big( \U_\X^T \U_\X \big) - \sigma_i (\U_\X^T \S^T \S \U_\X) \big| \nonumber \\
&\; \leq \;&
\sigma_{\max} \big( \U_\X^T \U_\X - \U_\X^T \S^T \S^T \U_\X \big) \nonumber \\
&\; = \;&
\big\| \U_\X^T \U_\X - \U_\X^T \S^T \S^T \U_\X \big\|_2 . \nonumber
\end{eqnarray}
Since the leverage scores of $\X$ are also the leverage scores of $\U_\X$,
it follows from Lemma~\ref{lem:thm7_drineas2008cur} that
\[
\EB \big\| \U_\X^T \U_X - \U_\X^T \S^T \S^T \U_X \big\|_2
\; \leq \;
\OM \bigg( \sqrt{\frac{\ln c}{c} } \bigg)  \: \| \U_\X \|_F \, \| \U_\X \|_2
\; = \; \OM \bigg( \sqrt{\frac{d \ln c}{c}}  \bigg) .
\]
It then follows from Markov's inequality that the inequality
\[
\big| 1 - \sigma_i^2 (\S \U_X) \big|
\; \leq \;
\delta_1^{-1} \OM \bigg( \sqrt{\frac{d\ln c}{c} }  \bigg)
\]
holds with probability at least $1-\delta_1$.
When
\begin{eqnarray} \label{eq:proof:thm:leverage_score_sampling:c1}
c \geq \OM ( d \delta_1^{-2} \epsilon_1^{-2} \ln ( d \delta_1^{-2} \epsilon_1^{-2})) ,
\end{eqnarray}
the inequality
\begin{eqnarray} \label{eq:proof:thm:leverage_score_sampling:1}
\sigma_{\min}^2 (\S \U_X) \geq 1 - \epsilon_1
\end{eqnarray}
holds with probability at least $1-\delta_1$.

Now we bound the term $ \| \U_\X^T \S^T \S \U_\X^{\perp} {\U_\X^{\perp}}^T \y \|_2$.
Since $\U_\X^T {\U_{\X}^{\perp}}^T = \0$, we have that
\begin{eqnarray}
\Big\| \U_\X^T \S^T \S \U_\X^{\perp} {\U_\X^{\perp}}^T \y \Big\|_2
&\; = \;& \Big\| \big(\U_\X^T  \big) \big( \U_\X^{\perp} {\U_\X^{\perp}}^T \y \big)
-  \big( \U_\X^T \big) \S^T \S  \big( \U_\X^{\perp} {\U_\X^{\perp}}^T \y \big) \Big\|_2 . \nonumber
\end{eqnarray}
Since the leverage scores of $\X$ are also the leverage scores of $\U_\X$,
it follows from Lemma~\ref{lem:thm7_drineas2008cur} that
\begin{eqnarray}
&\EB\Big\| \big( \U_\X^T  \big) \big( \U_\X^{\perp} {\U_\X^{\perp}}^T \y \big)
-  \big( \U_\X^T \big) \S^T \S  \big( \U_\X^{\perp} {\U_\X^{\perp}}^T \y \big) \Big\|_2\nonumber\\
& \;\leq\;
\frac{1}{\sqrt{c}} \;
\Big\|\U_\X \Big\|_F \, \Big\| \U_\X^{\perp} {\U_\X^{\perp}}^T \y \Big\|_2
\; = \;
\sqrt{\frac{d}{c}} \; \Big\| \U_\X^{\perp} {\U_\X^{\perp}}^T \y \Big\|_2 . \nonumber
\end{eqnarray}
It follows from the Markov's inequality that the following inequality holds with probability at least $1 - \delta_2$:
\begin{eqnarray} \label{eq:proof:thm:leverage_score_sampling:2}
\Big\| \U_\X^T \S^T \S \U_\X^{\perp} {\U_\X^{\perp}}^T \y \Big \|_2
\; \leq \;
\frac{\delta_2^{-1} \sqrt{d}}{\sqrt{c}} \Big\| \U_\X^{\perp} {\U_\X^{\perp}}^T \y \Big\|_2 .
\end{eqnarray}

Thus when
\begin{eqnarray} \label{eq:proof:thm:leverage_score_sampling:c2}
c \geq d \delta_2^{-2} \epsilon_2^{-2} (1-\epsilon_1)^{-2} ,
\end{eqnarray}
it follows from (\ref{eq:proof:thm:leverage_score_sampling:1}), (\ref{eq:proof:thm:leverage_score_sampling:2}),
and the union bound that the inequality
\begin{eqnarray} \label{eq:proof:thm:leverage_score_sampling:3}
\frac{\big\| \U_\X^T \S^T \S \U_\X^{\perp} {\U_\X^{\perp}}^T \y \big \|_2}{ \sigma_{\min}^2 (\S \U_\X) }
\; \leq \;
\epsilon_2 \Big\| \U_\X^{\perp} {\U_\X^{\perp}}^T \y \Big\|_2
\end{eqnarray}
holds with probability at least $1 - \delta_1 - \delta_2$.
We let $\epsilon_1 = 0.5$, $\epsilon_2 = \sqrt{\epsilon}$, $\delta_1 = \delta_2 = 0.1$,
and let $\z_\S$ be defined in Lemma~\ref{lem:lem1_drineas2011faster}.
When
\[
c \geq \max \big\{ \OM(d \ln d) , \, 400 d \epsilon^{-1} \big\} ,
\]
it follows from (\ref{eq:proof:thm:leverage_score_sampling:c1}), (\ref{eq:proof:thm:leverage_score_sampling:c2}),
(\ref{eq:proof:thm:leverage_score_sampling:3}), and Lemma~\ref{lem:lem1_drineas2011faster} that
with probability at least $0.8$ the following inequality holds:
\[
\| \z_\S \|_2
\; \leq \;
\sqrt{\epsilon} \Big\| \U_\X^{\perp} {\U_\X^{\perp}}^T \y \Big\|_2 .
\]
Since $\U_\X^{\perp} {\U_\X^{\perp}}^T \y = \y - \X \bet_{\lsr}$,
the theorem follows directly from Lemma~\ref{lem:lem1_drineas2011faster}.
\end{proof}

\subsection{Proof of Theorem~\ref{thm:uniform_sampling}} \label{sec:proof:thm:uniform_sampling}

\begin{proof}
We first follow some of the techniques of \cite{gittens2011spectral} to bound the two terms
\begin{eqnarray}
\sigma_{\max}^2 (\S \U_\X) &\; = \;& \sigma_{\max} (\U_\X^T \S^T \S \U_\X) , \nonumber \\
\sigma_{\min}^2 (\S \U_\X) &\; = \;& \sigma_{\min} (\U_\X^T \S^T \S \U_\X) . \nonumber
\end{eqnarray}
We let $\u_i \in \RB^d$ be the $i$-th column of $\U_\X^T$,
and let $\W_1 , \cdots ,  \W_c$ be $d\times d$ matrices sampled i.i.d.\ from
$\big\{ \u_i \u_i^T \big\}_{i=1}^{n}$ uniformly at random without replacement.
Obviously, $\sigma_{k} \big(\U_\X^T \S^T \S \U_\X\big) = \sigma_{k} \big(\sum_{j=1}^c \W_j \big)$.
We accordingly define
\[
R = \max_{j} \lambda_{\max} (\W_j) = \max_{i} \big\| \u_i \big\|_2^2 = \frac{d}{n} \mu \textrm{,}
\]
where $\mu$ is the row matrix coherence of $\X$, and define
\begin{eqnarray}
\xi_{\min}
&\; = \; & c  \lambda_{\min}\big( \EB \W_1 \big)
\; = \;  \frac{c}{n} \lambda_{\min}\Big( \U_\X^T \U_\X \Big)
    \;=\; \frac{c}{n} , \nonumber \\
\xi_{\max}
&\; = \; & c  \lambda_{\max}\big( \EB \W_1 \big)
\; = \;  \frac{c}{n} \lambda_{\max}\Big( \U_\X^T \U_\X \Big)
    \;=\; \frac{c}{n} . \nonumber
\end{eqnarray}
Then we apply Lemma~\ref{lemma:matrix_chernoff} and obtained the following inequality:
\begin{eqnarray}
\PB\bigg[ \lambda_{\min} \Big(  \sum_{i=1}^c \W_i \Big) \leq \frac{\theta_1 c}{n}\bigg]
& \leq & d \bigg[ \frac{e^{\theta_1 -1}}{\theta_1^{\theta_1}} \bigg]^{\frac{c}{d \mu}}
\;\triangleq \; \delta_1 , \nonumber \\
\PB\bigg[ \lambda_{\max} \Big(  \sum_{i=1}^c \W_i \Big) \geq \frac{\theta_2 c}{n}\bigg]
& \leq & d \bigg[ \frac{e^{\theta_2 -1}}{\theta_2^{\theta_2}} \bigg]^{\frac{c}{d \mu}}
\;\triangleq \; \delta_2 , \nonumber
\end{eqnarray}
where $\theta_1 \in (0, 1)$, $\theta_2 > 1$, and $\delta_1 , \delta_2 \in (0, 1)$ are arbitrary real numbers.
We set
\begin{equation}\label{eq:theorem:uniform_sampling:c}
c \;=\; \max \bigg\{ \frac{\mu d \ln (d / \delta_1)}{\theta_1 \ln \theta_1 - \theta_1 + 1} ,\;
\frac{\mu d \ln (d / \delta_2)}{\theta_2 \ln \theta_2 - \theta_2 + 1}
\bigg\},
\end{equation}
it then follows that with probability at least $1 - \delta_1 - \delta_2$, both of the following two inequalities hold:
\begin{equation}\label{eq:theorem:uniform_sampling:sigma}
\sigma_{\max} \Big( \S \U_\X \Big) \leq \sqrt{\frac{\theta_2 c }{ n}}
\qquad \textrm{ and } \qquad
\sigma_{\min}^{-2} \Big( \S \U_\X \Big) \leq \frac{n}{\theta_1 c}.
\end{equation}

Now we seek to bound the term $\big\| \S \U_\X^{\perp} {\U_\X^{\perp}}^T \y \big\|_2$.
Let $\CM$ be an index set of cardinality $c$ with each element chosen from $[n]$ uniformly at random without replacement,
and let $\y^{\perp} = \U_\X^{\perp} {\U_\X^{\perp}}^T \y$,
then we have that
\begin{eqnarray}
\EB \big\| \S \U_\X^{\perp} {\U_\X^{\perp}}^T \y \big\|_2^2
\; = \; \EB \big\| \S \y^{\perp} \big\|_2^2
\; = \; \EB \sum_{j \in \CM} (y^{\perp}_j)^2
\; = \; \frac{c}{n} \big\| \y^{\perp} \big\|_2^2. \nonumber
\end{eqnarray}
Thus with probability at least $1 - \delta_3$ the following inequality holds:
\begin{eqnarray}\label{eq:theorem:uniform_sampling:residual}
\big\| \S \U_\X^{\perp} {\U_\X^{\perp}}^T \y \big\|_2^2
\; \leq \; \frac{c}{n \delta_3 } \big\| \U_\X^{\perp} {\U_\X^{\perp}}^T \y \big\|_2^2.
\end{eqnarray}

Finally, it follows from inequalities (\ref{eq:theorem:uniform_sampling:sigma}, \ref{eq:theorem:uniform_sampling:residual})
and Lemma~\ref{lem:lem1_drineas2011faster}
that
\begin{eqnarray*}
\|\z_{\S} \|_2
& \;\leq\; & \frac{ \big\| \U_\X^T \S^T \S \U_\X^{\perp} {\U_\X^{\perp}}^T \y \big\|_2 }{ \sigma_{\min}^2 (\S \U_\X) }
\; \leq \; \frac{ \big\| \U_\X^T \S^T \big\|_2 \, \big\| \S \U_\X^{\perp} {\U_\X^{\perp}}^T \y \big\|_2 }{ \sigma_{\min}^2 (\S \U_\X) } \\
&\;\leq\;& \sqrt{\frac{\theta_2 c }{ n}} \: \frac{n}{\theta_1 c} \sqrt{\frac{c}{n \delta_3} } \, \big\| \U_\X^{\perp} {\U_\X^{\perp}}^T \y \big\|_2
\; = \; \frac{1}{\theta_1} \sqrt{\frac{\theta_2}{\delta_3}} \, \big\| \U_\X^{\perp} {\U_\X^{\perp}}^T \y \big\|_2.
\end{eqnarray*}
Here the first two inequalities hold deterministically, and the third inequality holds with probability at least $1 - \delta_1 - \delta_2 - \delta_3$.

We set $\theta_1 = 1-\epsilon$, $\theta_2 = 1 + \epsilon$, $\delta_1 = \delta_2 = \delta$, and $\delta_3 = 1 - 3 \delta$.
Since $\ln \theta \approx \theta - 1$ when $\theta$ is close to $1$,
it follows from (\ref{eq:theorem:uniform_sampling:c}) that when $c > \mu d \epsilon^{-2} (\ln d - \ln \delta)$,
the inequality
\[
\big\| \z_\S \big\|_2^2 \leq \frac{{1+\epsilon}}{(1-\epsilon)^2 (1-3\delta)} \, \Big\|\U_\X^{\perp} {\U_\X^{\perp}}^T \y \Big\|_2^2
\]
holds with probability at least $\delta$.

Setting $\theta_1 = 0.9556$, $\theta_2 = 1.045$, $\delta_1 = \delta_2 = 0.0015$, and $\delta_3 = 0.947$,
we conclude that when $c = 1000 \mu d (\ln d + 7)$, the inequality
$\| \z_\S \|_2^2 \leq 1.2 \, \big\|\U_\X^{\perp} {\U_\X^{\perp}}^T \y \big\|_2^2$
holds with probability at least $0.05$.
Then the theorem follows directly from Lemma~\ref{lem:lem1_drineas2011faster}.
\end{proof}

\bibliographystyle{abbrv}
\bibliography{LevErrorBound}

\end{document}